\documentclass[11pt, lefttitle]{behroozarxiv}

\usepackage[T1]{fontenc}
\usepackage[utf8]{inputenc}

\usepackage{enumitem}
 
\usepackage{amsmath,amssymb,amsfonts,mathtools}

\usepackage{amsthm}

\newtheorem{remark}{Remark}
\newtheorem{proposition}{Proposition}

\title{Principles Do Not Apply Themselves: A Hermeneutic Perspective on AI Alignment}

\DeclareRobustCommand{\orcidicon}{%
  \href{https://orcid.org/0000-0001-9568-4166}{%
    \raisebox{-0.2ex}{\includegraphics[height=1.6ex]{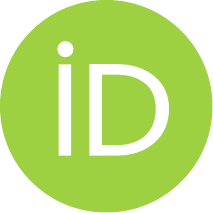}}%
  }%
}

\author{Behrooz Razeghi~\orcidicon}
\affiliation{School of Engineering and Applied Sciences, Harvard University}

\abstract{
AI alignment is often framed as the task of ensuring that an AI system follows a set of stated principles or human preferences, but general principles rarely determine their own application in concrete cases. When principles conflict, when they are too broad to settle a situation, or when the relevant facts are unclear, an additional act of judgment is required. This paper analyzes that step through the lens of hermeneutics and argues that alignment therefore includes an interpretive component: it involves context-sensitive judgments about how principles should be read, applied, and prioritized in practice. We connect this claim to recent empirical findings showing that a substantial portion of preference-labeling data falls into cases of principle conflict or indifference, where the principle set does not uniquely determine a decision. We then draw an operational consequence: because such judgments are expressed in behavior, many alignment-relevant choices appear only in the distribution of responses a model generates at deployment time. To formalize this point, we distinguish deployment-induced and corpus-induced evaluation and show that off-policy audits can fail to capture alignment-relevant failures when the two response distributions differ.
We argue that principle-specified alignment includes a context-dependent interpretive component.
}

\preprint{arXiv preprint}
\date{\today}

\begin{document}


\maketitle


\section{Introduction}
\label{sec:intro}

Much work on AI alignment begins from a familiar picture. One specifies a set of high-level aims---for example, that a system should be helpful, honest, and harmless---and then trains the system, often on preference data and through reinforcement learning, to produce outputs that conform to them \citep{ouyang2022training, bai2022training, bai2022constitutional, buyl2025discretion}. That picture captures part of the problem, but it leaves open how those principles are to be applied in concrete cases. Stating a principle is not the same as determining how it applies in a particular case. A principle may be clearly stated and still leave open which features of the situation matter, what it requires in the case at hand, and how it should be weighed against competing considerations. When a case is ambiguous, when several principles point in different directions, or when the relevance of the facts is itself disputed, the principle set does not by itself yield a unique answer \citep{sorensen2024value, conitzer2024social}. This paper studies that gap through the lens of hermeneutics \citep{caputo2024alignment}. In the limited sense relevant here, hermeneutics concerns how general expressions, including rules and principles, acquire determinate force in particular cases \citep{gadamer2013truth}.

This difficulty already appears in current alignment practice. Buyl~\textit{et al.}~\citep{buyl2025discretion} show that when alignment data are built from pairwise preference judgments, many cases are not settled by the stated principles alone. Some are \textit{consensus} cases, in which the principles support the same judgment. But many others are cases of \textit{conflict}, in which different principles support different judgments, or \textit{indifference}, in which the principle set does not determine a unique preference. In those cases, the final label cannot be derived from the principle set itself. A further judgment must therefore enter somewhere: in annotator instructions, tie-breaking rules, aggregation procedures, or later stages of the pipeline. The point is not only that labels may outrun principles. It is also that some alignment-relevant judgments may be introduced outside the principle list itself. A dataset may therefore look principle-aligned at the level of its final labels even though some of the judgments that fixed those labels were made elsewhere and cannot be recovered from the labels alone. Related work likewise suggests that alignment data cannot be treated as the transparent expression of a single, settled value order. Recent studies emphasize the plural, subjective, and multicultural character of human feedback, and show that whose feedback is collected can materially affect what alignment data represent \citep{sorensen2024value, kirk2024prism, conitzer2024social}.

\vspace{4pt}

Once this point is taken seriously, it has an immediate implication for evaluation. If principle application depends on context-sensitive judgment, then many alignment-relevant choices are expressed not only in a final preference label, but in behavior itself. Faced with an underspecified request, a system may ask a clarifying question, proceed with a tentative answer, refuse to respond, or provide a constrained alternative. Faced with competing principles, it may privilege one while suppressing another. These are not merely stylistic differences. They are part of how the governing principles are applied in practice. For language models in particular, such choices are realized through the model's conditional distribution over responses given a prompt.

\vspace{4pt}

This observation leads to a basic distinction between two evaluation regimes. At deployment, behavior is produced by the model's own distribution over responses. By contrast, many benchmarks and preference datasets evaluate the model on fixed candidate responses drawn from a corpus that was constructed independently of the model under test. These two regimes need not coincide. An off-policy evaluation based on static candidates can therefore fail to capture errors that arise only when the model generates its own continuations at deployment time. We formalize this point by comparing deployment-induced and corpus-induced evaluation, and by stating a simple total-variation bound on the resulting risk gap. The purpose of this formalization is modest: it identifies a principled sense in which static preference audits can provide only a surrogate view of alignment when the relevant behavior is generated on-policy. The formal result does not attempt to model interpretation itself; it identifies a distributional limitation of static audits when alignment-relevant choices are expressed in generated behavior.

\vspace{4pt}

While the interpretive character of alignment is visible even at the level of single outputs, it becomes especially clear in multi-turn dialogue \citep{klassen2024pluralistic, kirk2024prism}. Successive turns are not independent of one another. The interpretation of a later input is shaped by earlier outputs, local resolutions of conflict at one turn can constrain the options available at later turns, and a principle that seemed straightforward in one part of the exchange may need to be reconsidered once the practical stakes of applying it become clearer. If multi-turn alignment is studied only as a sequence of isolated prompt--response pairs, these cross-turn dependencies are difficult to describe and harder still to audit.

\begin{figure}[t!]
    \centering
    \includegraphics[width=\linewidth]{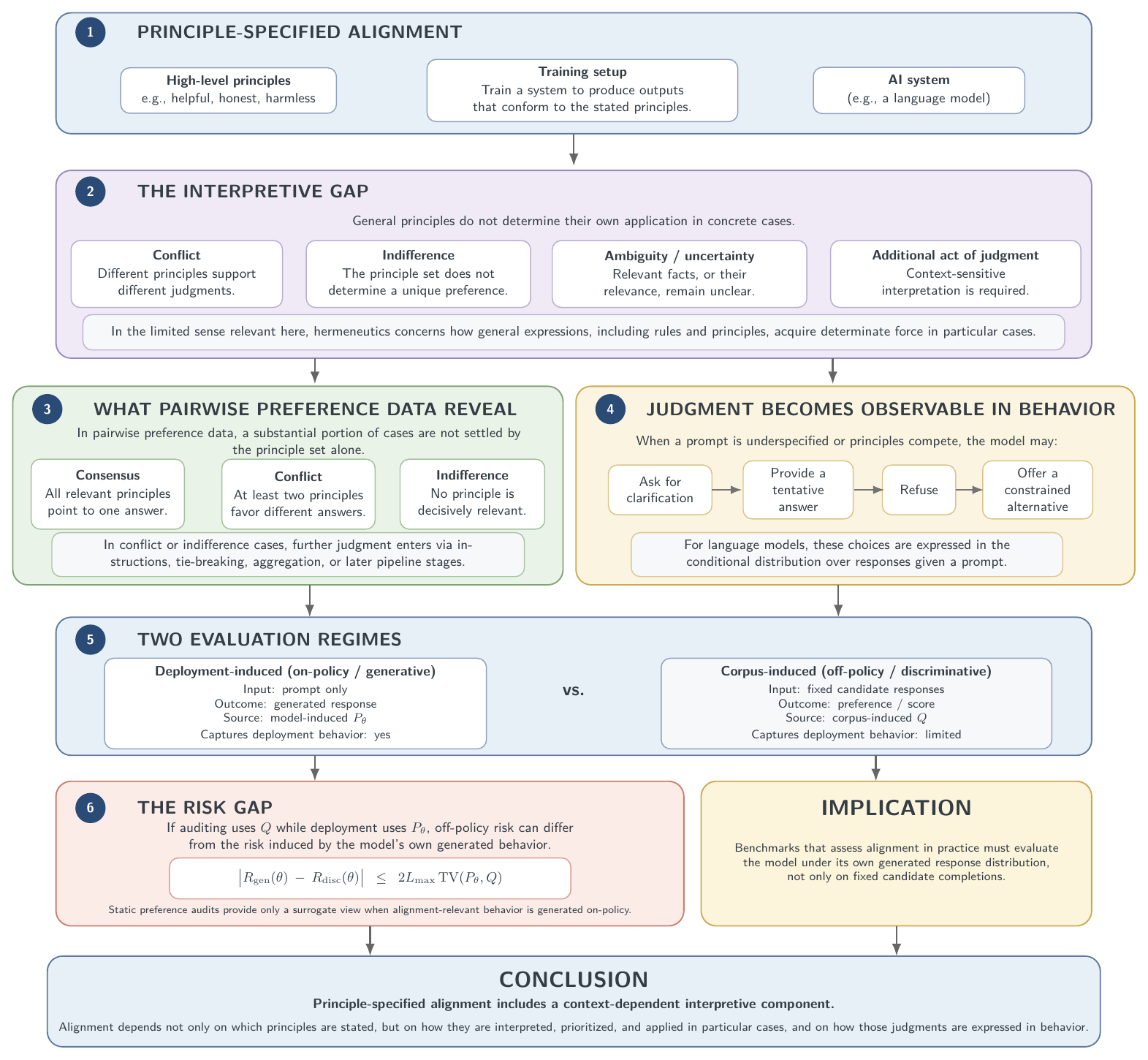}
    \caption{Conceptual structure of principle-specified alignment and its evaluative consequences.}
    \label{fig:conceptual_overview}
\end{figure}

\paragraph{``Understanding'' as a Hermeneutical Event.}

From the beginnings of biblical and classical exegesis, hermeneutics has dealt with the question of how we come to \textit{understand} a text, situation, or work of art. Initially, the question of understanding was treated as a methodological issue, concerned with recovering the author's intention and thereby arriving at a stable interpretation. But as hermeneutics progressed, it increasingly became a question of the historical and dialogical character of interpretation. We never approach a text \textit{tabula rasa} (as a blank slate). Rather, we bring our own preconceptions and our cultural and linguistic contexts with us \citep{heidegger2002time, gadamer2013truth}, and these shape what we are able to hear, what questions we are able to ask, and what we are able to learn. Understanding a text is therefore not just a matter of extracting information. Rather, it is a \textit{transformative event} in which both the subject matter and the reader's \textit{horizon of understanding} are changed. It is precisely this relation, known as the \textit{hermeneutic circle}\footnote{%
The expression \textit{hermeneutische Cirkel} appears in the classical philologist August Boeckh's 1809 lectures on philological method; Schleiermacher later develops the part--whole circularity of understanding without adopting the expression ``hermeneutical circle.'' See \citep{boeckh1886enzyklopadie} and \citep{grondin2015hermeneutical}.}), that lies at the center of the hermeneutic account of understanding. An interpretation of a \textit{part} of a whole can only be made by reference to some prior idea of what the \textit{whole} must be. But at the same time, the interpretation of the whole in light of the part must lead to a revision of our original understanding of that whole, and conversely. The depth of our understanding therefore depends on the historically conditioned openness of the dialogue in which we engage with the text and with it as a whole. It is in this engagement with the material that our preconceptions are challenged and that we come to see things in a new light. Understanding is therefore an ongoing dialogue in which we continually test our interpretations against tradition, context, and the material itself, and through this process come to see things more clearly.

\paragraph{Principle Application Requires Interpretation.}

A general principle may be clearly stated and yet still fail, on its own, to determine how it should be applied in a particular case. It may leave open which features of the case are relevant, how broadly it should be interpreted under the circumstances, and how it should be weighed against other applicable principles. In such cases, the move from principle to judgment is not automatic.

The question now arises as to whether and how a hermeneutic approach can shed light on this matter. In the hermeneutic tradition, application is not seen as the formal, mechanical application of a principle whose content has already been securely established. Rather, understanding always begins from pre-understanding and within a horizon that is itself already constituted \citep{heidegger2002time, gadamer2013truth}. What a principle demands is determined only in relation to the particular case. For that reason, application is not a formal procedure carried out on a text or principle in a context-independent manner once its content has been established. It is context-dependent and can remain indeterminate even when the principle itself is determinate. It is therefore possible to adhere to the same principle while reaching different reasonable judgments in different cases, because its scope, weight, or relation to other principles may remain open.

If alignment principles are intended to serve as general normative guidelines, they must be interpretable in sufficient detail to yield determinate judgments in specific cases. This may be necessary because principles can conflict with one another, because they may be incomplete or imprecise, or because they do not suffice to determine which of several permitted responses should be selected. The point of this paper is not to undermine formal approaches to alignment, but to note that the transition from abstract principles to particular judgments requires interpretation in one form or another.

Figure~\ref{fig:conceptual_overview} summarizes the conceptual structure of the paper, from the interpretive underdetermination of principle-specified alignment to its behavioral and evaluative consequences.

\paragraph{Contributions.}
The contribution of the paper is accordingly twofold. First, it gives a philosophically precise account of why principle-specified alignment includes an interpretive stage rather than reducing entirely to the mechanical execution of a fixed rule set. Second, it draws an operational consequence from that account by showing why evaluation must attend to the data-generating distribution under which alignment is assessed, and by formalizing the gap between deployment-induced and corpus-induced audit regimes.

\paragraph{Paper Organization.}
Section~\ref{sec:alignment-as-interpretation} applies the preceding claim about the interpretation of general principles to contemporary alignment systems and discusses principle conflict, indifference, and discretion. Section~\ref{sec:behavior-evaluation} then connects this interpretive claim to observable model behavior, and Section~\ref{ssec:on-off-policy} formalizes the distinction between on-policy and off-policy evaluation and derives a bound on their discrepancy.

\section{Alignment as an Interpretive Process}
\label{sec:alignment-as-interpretation}

Alignment is formally described through a range of algorithmic pipelines, from language-model training to post hoc fine-tuning through Reinforcement Learning from Human Feedback (RLHF) \citep{christiano2017deep, ouyang2022training, bai2022training, bai2022constitutional, buyl2025discretion}. Yet recent work shows that substantial discretion remains at different stages of these pipelines. In particular, annotators---whether human or model-based---must decide whether one candidate response is better than another. That decision is not always fixed by the stated principle set itself, because the relevant principles may conflict or may fail to provide sufficient guidance in a particular case. It is difficult to formulate principles such as safety and helpfulness in a way that is both broadly applicable and fully decisive in every instance. Hermeneutics draws attention to the same difficulty: when a rule or norm is general, its application to a concrete case can still remain open.

The fact that the outcome is reduced to a single better/worse label means that we often cannot see whether, or how, one principle was balanced against others that may also have been relevant. A system may therefore be described as aligned with a principle set while, in practice, some principles are overridden, underweighted, or left operationally unspecified.

From a hermeneutic perspective, providing the system with a set of rules such as ``avoid harm'' and ``be helpful'' is not sufficient to determine their application in every case. Instead, alignment is worked out on a case-by-case basis, as higher-level principles are applied in light of the user's query, broader socio-historical context, and the discretionary choices made in particular cases. From a Gadamerian perspective, no principle is interpreted in isolation. Instead, it is interpreted against a shifting horizon of understanding that may change over the course of a conversation.

\subsection{Principle Conflicts and Alignment Discretion}
\label{ssec:conflict}

In \cite{buyl2025discretion}, the authors distinguish three categories of principle agreement in alignment data: \textit{consensus}, in which all relevant principles point unambiguously to one answer; \textit{conflict}, in which at least two principles disagree on which answer is preferred; and \textit{indifference}, in which no principle is decisively relevant. When either \textit{conflict} or \textit{indifference} arises, the aligner, whether a human or a model, must \textit{exercise discretion} by deciding which principle to override, or whether none of the listed principles really matters. In the hermeneutical sense, the user or system \textit{fuses horizons} in a situation not fully determined by the context itself and must choose which dimension of meaning is to be taken as pertinent.

The crucial observation of Buyl~\textit{et~al.} is that a large portion of the data used for training alignment models consists of examples in the conflict or indifference categories---more than one might intuitively assume. They also show that disagreement between annotators exists even in cases where there is supposed to be a clear principle consensus; they refer to this as ``\textit{discretion arbitrariness}''. From a hermeneutic perspective, we would more likely describe such cases as instances of \textit{misreading} or \textit{ignoring} the text that constrains the intended meaning of the alignment principles. In this way, we see that there can be a significant difference between the principles that are written down and the principles that are actually enforced in practice.

\subsection{Annotator Judgment as Hermeneutic Event} 

Far from being a mechanistic procedure, each preference decision in alignment training can be viewed as a miniature \textit{hermeneutic event}. The annotator (or AI-based annotator) must interpret the user's query and context, map it onto a set of broad alignment principles that may be ambiguous, vague, or in conflict, resolve any such conflicts by deciding which principle takes priority in the case at hand, and finally render judgment by selecting one candidate answer over the other. As in Gadamerian or legal hermeneutics, part of the alignment challenge is that principles like ``avoid harm'' or ``respect free speech'' do not specify in advance how they should be traded off in borderline scenarios. The difficulty must therefore be handled through ongoing interpretive judgment.

Buyl \textit{et al.} \citep{buyl2025discretion} also examine \textit{principle supremacy}, that is, how often one principle is favored over another in pairwise conflicts. Their analysis shows that the practical priority of principles can differ across datasets and annotators. In cases where annotators repeatedly favor ``be helpful'' over a competing principle, this suggests that helpfulness is being treated as higher priority in practice, even if that ordering is not made explicit in the official list of principles \citep{huang2024collective}.

\subsection{The Illusion of Static Alignment}

One might think that listing a set of alignment rules in a ``Constitutional AI'' system \citep{bai2022constitutional} or in a company's ``AI Usage Guidelines'' would be enough to determine the normative stance of the system. But from a hermeneutic perspective, this interpretive gap cannot be eliminated by enumeration alone. It is difficult to specify rules in advance in a way that fully closes this gap, because the system must still determine which principle takes precedence when there is a conflict, or what to do when none of the listed principles seems clearly applicable. Even if fundamental rights and values are explicitly listed, there is no guarantee that they will be given the same priority in practice across all cases.

So an \textit{illusion of alignment} may arise: the AI is trained on data that appears to embody all the relevant principles; however, the fact that conflicts or principle-indifference frequently arise means that \textit{annotators} hold the real decision-making power. Buyl \textit{et al.} call this effect ``alignment-washing''. What they discuss is related to a well-known warning in hermeneutics: if interpreters are able to make discretionary choices behind the veil of yes/no answers, then it becomes impossible to tell whether the system really embodies the ``principles'' it is said to implement.

\subsection{Toward Hermeneutic Transparency and Accountability}

Hermeneutics has long insisted that interpretation must remain open to review. In legal systems, the exercise of discretion is acceptable only so long as it is structured, reviewable, and justified (as evidenced by judicial reasoning in court decisions). Buyl \textit{et al.} then propose metrics---\textit{discretion arbitrariness}, \textit{principle supremacy}, and \textit{discretion discrepancy}---through which it can be seen when, how, and how consistently an AI system or its human trainers apply the same interpretive framework to principle-laden judgments.
Such transparency matters for two reasons:\vspace{-2pt}
\begin{enumerate}[leftmargin=*] 
    \item[-] 
    \textit{Interpretive Trust:} If end-users or the broader public cannot see how conflicts in alignment rules were resolved, they cannot meaningfully trust that the system is respecting the stated principles.\vspace{-2pt} 
    \item[-] 
    \textit{Feedback and Iteration:} By making discretionary choices measurable, developers can refine or correct them—identifying if certain fundamental values (like ``avoid racism'') are frequently being overridden by less critical concerns (e.g. be ``obedient''). 
\end{enumerate} 

In hermeneutic terms, a \textit{self-reflexive} approach to alignment sees these metrics as akin to textual commentaries or legal precedents, \textit{marking the shifts} in how the system weighs and fuses different horizons. A hermeneutic approach to alignment demands:\vspace{-2pt}
\begin{itemize}[leftmargin=*]
    \item[-]
    \textit{Context-awareness:} 
    The ability to refine or ``\textit{instantiate}'' general principles into more specific sub-principles as contexts vary.\vspace{-2pt}
    \item[-]
    \textit{Iterative dialogue:} 
    Alignment must allow for multi-turn interactions, so that newly revealed ambiguities or conflicts can be reconsidered rather than dismissed.\vspace{-2pt}
    \item[-]
    \textit{Procedural logs or rationale:} 
    As in a judge's opinion, an AI system (or its human trainers) should keep track of \textit{why} it chose one principle override over another, enabling external review and improvement.
\end{itemize}

\subsection{Beyond Fixed Rule-Sets: Alignment as Ongoing Dialogue} 

A hermeneutic perspective is quite helpful in showing that alignment cannot be a one-time event of rule injection from on high. As Buyl \textit{et al.}'s empirical findings illustrate, principle conflicts can easily add up and disagreements among annotators can arise even over apparently clear-cut textual consensus. This shows that alignment is a \textit{continuing conversation}, between the user and the system on the one hand, and among the various stakeholders, implementers, and human annotators of the system on the other.
This point also connects with recent work on pluralistic alignment. Sorensen \textit{et al.} emphasize that human values, rights, and duties may stand in tension rather than forming a single unified order \citep{sorensen2024value}. Kirk \textit{et al.} show that alignment judgments vary across participants and cultural settings, so that the composition of the feedback population matters to what an alignment dataset represents \citep{kirk2024prism}. Klassen \textit{et al.} further argue that pluralistic alignment must often be understood over time rather than at a single decision point \citep{klassen2024pluralistic}. Conitzer \textit{et al.} similarly frame the problem as one of aggregating diverse human feedback rather than assuming that a single preference order is already given \citep{conitzer2024social}. 
Our claim is compatible with these developments, but different in emphasis: we focus on the interpretive step by which principles are applied in particular cases, and on the fact that this step becomes behaviorally visible only under the model's own generated trajectories.

Whereas Gadamer's hermeneutics is often taken to imply that texts do not interpret themselves, we argue that alignment rules do not implement themselves either. The interpretive circle of context, questions, and countervailing values that continuously feed back into each other must be \textit{acknowledged} and \textit{structured}, rather than \textit{concealed}. We take hermeneutics to suggest that it is by dealing with the ambiguities or contestabilities of meaning that understanding is achieved. In the case of AI alignment, the task is not to eliminate these ambiguities or contestabilities, but to give them some structure, so that discretionary judgments are traceable, consistent, and open to critique.

\section{From Hermeneutic Discretion to Observable Behavior}
\label{sec:behavior-evaluation}
 
The hermeneutic claim in Sec.~\ref{ssec:conflict} is not only interpretive in a literary sense.
This has also an ``\textit{observational}'' consequence for the practice of alignment auditing. The underdetermination of action by principles in cases of conflict, indifference, or borderline cases means that the aligner must make an additional decision (about which consideration is to be treated as decisive, about requesting clarification, about refusing, and about the wording of the response). These decisions are not hidden mental states of the model; they manifest as the model's ``\textit{behavior}'': specifically, they are the choice of a conditional distribution over responses, or equivalently, the choice of a kernel $x\mapsto \pi_{\boldsymbol\theta}(\cdot\mid x)$ that governs which continuations are more or less likely to occur.

This observation matters because an audit can only measure what its data-generating process exposes.
Paired-preference corpora provide a prompt $x$ together with a fixed set of candidate completions $(y_0,y_1)$ and a label, but they do not sample the model's free-generation distribution $\pi_{\boldsymbol\theta}(\cdot\mid x)$. Consequently, they cannot, in general, identify deployment-time quantities that depend on the model's own continuation law, including (i) the frequency with which the model chooses clarification over direct answering, (ii) the distribution of principle-tradeoffs manifested in its completions, and (iii) any multi-turn phenomena that depend on the model's own outputs.

The preceding discussion raises an evaluation question: \textit{under which data-generating process is alignment being assessed}?
A static paired-preference corpus evaluates a principle-laden loss on a distribution of candidate responses that is fixed by dataset construction,
whereas deployment-time behavior is generated by the model's own conditional law of continuations. To make this distinction precise, we now introduce a minimal probabilistic model of prompt--response generation.

\subsection{On-Policy vs. Off-Policy Alignment Evaluation}
\label{ssec:on-off-policy}

If alignment-relevant judgments are expressed in behavior, then evaluation depends on the distribution that generates the behavior being assessed. This is the operational point behind the preceding discussion. A model may choose to address an underspecified prompt by asking for clarification, by providing a direct answer, by refusing to answer or by providing a constrained alternative. These are not only style choices. Rather, they represent the practical realization of the underlying formalism. For language models, such choices are expressed through the model's conditional distribution over responses given a prompt.

Let $\rho$ denote a distribution over prompts $x \in \mathcal{X}$, and let $\pi_{\boldsymbol{\theta}}(\cdot \mid x)$ denote the response distribution of the model being evaluated. The corresponding deployment-time, or \textit{on-policy}, joint distribution is
\begin{equation}
P_{\boldsymbol{\theta}}(x,y) \coloneqq \rho(x)\,\pi_{\boldsymbol{\theta}}(y \mid x).
\label{eq:on_policy_joint_main}
\end{equation}
By contrast, many alignment benchmarks and preference datasets evaluate models on candidate responses drawn from a fixed corpus construction process that does not depend on the current model. If $q(\cdot \mid x)$ denotes that fixed response kernel, the corresponding \textit{off-policy} joint distribution is
\begin{equation}
Q(x,y) \coloneqq \rho(x)\,q(y \mid x).
\label{eq:off_policy_joint_main}
\end{equation}
The two distributions share the same prompt marginal $\rho$ but may differ substantially in their response behavior.

To compare these regimes, let $\ell : \mathcal{X}\times\mathcal{Y} \to [0,L_{\max}]$ be any bounded misalignment loss. We define the on-policy (generative) and off-policy (discriminative) risks by
\begin{align}
\mathcal{R}_{\mathrm{gen}}(\boldsymbol{\theta})
&\coloneqq
\mathbb{E}_{(x,y)\sim P_{\boldsymbol{\theta}}}\bigl[\ell(x,y)\bigr],
\label{eq:rgen_main}
\\
\mathcal{R}_{\mathrm{disc}}(\boldsymbol{\theta})
&\coloneqq
\mathbb{E}_{(x,y)\sim Q}\bigl[\ell(x,y)\bigr].
\label{eq:rdisc_main}
\end{align}
The first quantity measures the risk induced by the model's own deployment behavior. The second evaluates the same loss on a fixed behavioral distribution supplied by the dataset.

The discrepancy between the two regimes is controlled by total variation distance. In particular,
\begin{equation}
\bigl|
\mathcal{R}_{\mathrm{gen}}(\boldsymbol{\theta})
- \mathcal{R}_{\mathrm{disc}}(\boldsymbol{\theta})
\bigr| \;\le\; 2L_{\max}\,\mathrm{TV}\!\bigl(P_{\boldsymbol{\theta}},Q\bigr),
\label{eq:tv_gap_main}
\end{equation}
where $\mathrm{TV}(P,Q)\coloneqq \tfrac12\int |\mathrm{d}P-\mathrm{d}Q|$.
A proof, together with a stronger worst-case characterization, is given in Appendix~\ref{appx:sec:on-off-policy}.

Equation~\eqref{eq:tv_gap_main} presents a fundamental limitation of static preference audits. It states that if the evaluation response distribution does not match the model's continuation distribution, then off-policy evaluation can differ from deployment-time risk. If a response distribution is used for auditing, and that distribution is not the same as the model's continuation distribution, then the audit is not measuring the same quantity as the model's deployment-time risk. There is a profound difference between auditing a model's behavior on fixed candidate completions and auditing the same model's behavior under free generation. The latter requires the model to generate completions from scratch, thereby instantiating the relevant principles in context, and it is these completions and their corresponding risk that are relevant in deployment.

This is especially important for the interpretive phenomena we have just discussed. Whether a model asks a clarifying question, declines to answer, softens a response, or prioritizes one principle over another is visible only in the trajectories generated by the model itself. Such phenomena therefore cannot, in general, be identified from a static off-policy corpus alone. This is not simply a pedantic observation about the corpus-induced and deployment-induced distinction. It highlights a deep difference between two notions of evaluation that are often treated as if they were the same. On the one hand, we have the evaluation of a fixed set of candidate outputs. On the other hand, we have the evaluation of the model's own way of applying principles in context.

Table~\ref{tab:on-vs-off-policy} summarizes the operational contrast between the two regimes. Off-policy evaluation observes model behavior only through a fixed response set drawn from a corpus-induced distribution, whereas on-policy evaluation measures behavior under the model's own deployment-induced response distribution.

\begin{table*}[t]
  \centering
  \caption{Off-policy versus on-policy alignment evaluation.}
  \label{tab:on-vs-off-policy}
  \begin{tabular}{@{}lcc@{}}
    \toprule
    & \textbf{Off-policy} (discriminative) & \textbf{On-policy} (generative) \\
    \midrule
    Evaluation input & fixed candidate responses & prompt only \\
    Observable outcome & preference / score & generated response \\
    Response source & corpus-induced $Q$ & model-induced $P_{\boldsymbol{\theta}}$ \\
    Typical data & paired preferences & roll-outs \\
    Captures deployment behavior & limited & yes \\
    \bottomrule
  \end{tabular}
\end{table*}

\section{Conclusion}

The paper has argued against understanding AI alignment as merely the application of a fixed set of principles, because general principles do not determine their own application in particular cases. It has done so by introducing a set of conceptual distinctions that make the act of judgment visible, and by using hermeneutics to show why this judgment should not be treated as unconstrained discretion. We have seen that recent findings of conflict, indifference, and discretion in alignment judgments are not isolated exceptions, but reflect a basic feature of a principle-based alignment framework: many decisions are not fixed by the stated principle set alone. We have also seen that evaluation cannot be limited to static preference labels alone. We conclude that AI alignment involves more than stating the right set of principles. An adequate account of AI alignment must address which principles are stipulated, how they are interpreted in the particular situation at hand, and how the judgments involved in that interpretation are expressed in the model's behavior over time. Future research can then examine in more detail how to make context-dependent judgment in alignment more visible, and thus more open to analysis and evaluation.

\subsubsection*{Acknowledgments}
This work was supported by the Swiss National Science Foundation (SNSF) under Grant No.~222339.

\bibliographystyle{tmlr}
\bibliography{references}

\appendix
\clearpage
\section{On-Policy and Off-Policy Evaluation}
\label{appx:sec:on-off-policy}

A possible question for alignment auditing is \emph{which behavioral distribution}
generates the text being evaluated: the model's own deployment behavior, or a
fixed dataset of candidate responses. This appendix states the formal result underlying the discussion in
Sec.~\ref{ssec:on-off-policy}.

Let $(\mathcal{X},\Sigma_{\mathcal{X}})$ be a measurable space of prompts and $(\mathcal{Y},\Sigma_{\mathcal{Y}})$ a measurable space of responses.
A stochastic language model (policy) with parameters $\boldsymbol{\theta}$ is a conditional distribution $\pi_{\boldsymbol{\theta}} \colon \mathcal{X} \times \mathcal{Y} \to [0,1]$ maps $(x,y)\mapsto \Pr_{\pi_{\boldsymbol{\theta}}}[y\mid x]$, written $y \sim \pi_{\boldsymbol{\theta}}( \cdot \!\mid\! x)$. 
Let $\rho$ be the (external) distribution\footnote{Empirically, $\rho$ is induced either by (i) drawing prompts from the live production feed, or (ii) randomly sampling prompts from an offline corpus.} of user prompts on $\mathcal{X}$. 
During inference the \textit{on-policy} joint distribution actually encountered by end-users is 
%
\begin{equation}
P_{\boldsymbol\theta}(\mathrm dx,\mathrm dy) \coloneqq \rho(\mathrm dx)\,\pi_{\boldsymbol\theta}(\mathrm dy\mid x),
%
\label{eq:joints_common_marginal_onpolicy_app}
\end{equation}
%
A roll-out therefore draws $(x,y) \sim P_{\boldsymbol{\theta}}$.
Most alignment benchmarks, in contrast, supply responses from a fixed distribution that does \emph{not} depend on~$\boldsymbol{\theta}$.

%
Fix any conditional kernel $q( \cdot \!\mid\! x)$ on $\mathcal{Y}$ and define the \textit{off-policy} joint law as 
\begin{equation}
Q(\mathrm dx,\mathrm dy) \coloneqq \rho(\mathrm dx)\,q(\mathrm dy\mid x).
\label{eq:joints_common_marginal_offpolicy_app}
\end{equation}
Hence $P_{\boldsymbol{\theta}}$ and $Q$ share the same prompt marginal~$\rho$ but may differ in their answer distributions. Classic paired-preference datasets such as \textsc{HH-RLHF} \citep{bai2022training_hh} or \textsc{PKU-SafeRLHF} \citep{ji2024pku} instantiate $Q$\footnote{Note that $Q$ is a fixed \textit{behavioral distribution} on $\mathcal{X} \times \mathcal{Y}$.} by keeping each prompt $x$ fixed and substituting two human-annotated responses $(y_{0},y_{1})$.

\begin{remark}[Equal–prompt–marginal assumption]\label{rmk:equal-marginal}
We assume the off-policy joint $Q(x,y)=\rho(x)\,q(y\mid x)$ preserves the \textit{same} prompt marginal as the on-policy joint $P_{\boldsymbol\theta}$: $\int_{\mathcal{Y}}\!P_{\boldsymbol\theta}(x,y)\, \mathrm{d}y \;=\; \int_{\mathcal{Y}}\!Q(x,y)\, \mathrm{d}y \;=\; \rho(x), \forall x\in\mathcal{X}$. The condition is automatically satisfied by paired-preference corpora such as HH-RLHF and PKU-SafeRLHF, which keep each prompt $x$ fixed and modify only the conditional answer distribution via human rewrites or model variants. Under this equality, the risk gap simplifies to the per-prompt difference in Prop.~\ref{prop:blind-spot-gap}.
\end{remark}

\paragraph{Generative vs. discriminative alignment risk.}
Let $\ell:\mathcal{X}\times\mathcal{Y}\!\to\! \mathbb{R}$ be a bounded mis-alignment loss (toxicity, principle violation, \textit{etc.}). We distinguish
\begin{subequations}
\begin{align}
  \mathcal{R}_{\text{gen}}^{\ell}(\boldsymbol{\theta})
  &\coloneqq \mathbb{E}_{(x,y)\sim P_{\boldsymbol{\theta}}}\bigl[\ell(x,y)\bigr],
    &\text{(on-policy risk)},
    \label{eq:r_gen}\\[2pt]
  \mathcal{R}_{\text{disc}}^{\ell}(\boldsymbol{\theta})
  &\coloneqq \mathbb{E}_{(x,y)\sim Q}\bigl[\ell(x,y)\bigr],
    &\text{(off-policy risk).}
    \label{eq:r_disc}
\end{align}
\end{subequations}
%
Indeed, Eq.~\ref{eq:r_gen} captures the \textit{generative} risk that end-users encounter, whereas Eq.~\ref{eq:r_disc} evaluates the same loss on a fixed \emph{behavioral distribution} $Q(x,y) = \rho(x)\,q(y \!\mid\! x)$, where $y\sim q( \cdot \!\mid\! x)$ is drawn off-policy rather than from $\pi_{\boldsymbol{\theta}}$.

\paragraph{Total-variation bound.}
Because $0 \le \ell \le L_{\max}$, a standard variational argument yields
\begin{equation}
\bigl\lvert\mathcal{R}_{\text{gen}}(\boldsymbol{\theta})
      -\mathcal{R}_{\text{disc}}(\boldsymbol{\theta})\bigr\rvert
\;\le\;
2L_{\max}\; \mathrm{TV}\!\bigl(P_{\boldsymbol{\theta}},Q\bigr),
\label{eq:tv-gap}
\end{equation}
where $\mathrm{TV}(P,Q) \coloneqq \tfrac12\int |\,\mathrm{d}P-\mathrm{d}Q\,|$. Unless $Q$ \textit{matches} the model's own footprint $P_{\boldsymbol{\theta}}$, Eq.~\ref{eq:tv-gap} shows that off-policy tests give only a loose surrogate guarantee.


\begin{proposition}[Blind-spot gap and its worst-case characterization]
\label{prop:blind-spot-gap}

Let $(\mathcal X,\Sigma_{\mathcal X})$ and $(\mathcal Y,\Sigma_{\mathcal Y})$ be measurable spaces and let $\rho$ be a probability measure on $\mathcal X$. Let $\pi_{\boldsymbol\theta}(\cdot\mid x)$ and $q(\cdot\mid x)$ be Markov kernels on $\mathcal Y$ given $x\in\mathcal X$.
Let $\ell:\mathcal X\times\mathcal Y\to\mathbb R$ be measurable and bounded, and define 
\begin{equation}
\Delta_\ell(\boldsymbol\theta)\coloneqq
\bigl|\mathcal R_{\mathrm{gen}}^{\ell}(\boldsymbol\theta)-\mathcal R_{\mathrm{disc}}^{\ell}(\boldsymbol\theta)\bigr|.
\end{equation}

\smallskip
\noindent\textbf{(i) Fixed-loss decomposition.}
For every bounded~$\ell$,
\begin{equation}
\label{eq:blind-gap}
\Delta_\ell(\boldsymbol\theta)
=
\Bigl|
\mathbb E_{x\sim\rho}
\Bigl[
\mathbb E_{y\sim\pi_{\boldsymbol\theta}(\cdot\mid x)}[\ell(x,y)]
-
\mathbb E_{y\sim q(\cdot\mid x)}[\ell(x,y)]
\Bigr]
\Bigr|.
\end{equation}
In particular, if $\pi_{\boldsymbol\theta}(\cdot\mid x)=q(\cdot\mid x)$ for $\rho$-almost every $x$, then $\Delta_\ell(\boldsymbol\theta)=0$
(for that fixed~$\ell$).

\smallskip
\noindent\textbf{(ii) Total-variation bound.}
If $0\le \ell(x,y)\le L_{\max}$ for all $(x,y)$, then
\begin{equation}
\Delta_\ell(\boldsymbol\theta)
\le
2L_{\max}\,\mathrm{TV}\!\bigl(P_{\boldsymbol\theta},Q\bigr),
\label{eq:tv_gap_app}
\end{equation}
where
\begin{equation}
\mathrm{TV}(\mu,\nu)
\coloneqq
\sup_{A\in\Sigma} |\mu(A)-\nu(A)|
=
\frac12\int |\mathrm d\mu-\mathrm d\nu|.
\label{eq:tv_def_app}
\end{equation}

\smallskip
\noindent\textbf{(iii) Worst-case characterization.}
Fix $L_{\max}>0$ and define the function class
\[
\mathcal L_\infty(L_{\max})
\coloneqq
\{\ell:\mathcal X\times\mathcal Y\to\mathbb R\ \text{measurable}:\ \|\ell\|_\infty\le L_{\max}\}.
\]
Define the \emph{worst-case} blind-spot gap
\begin{subequations}
\begin{align}
\label{eq:delta_infty_def_app}
\Delta_\infty(\boldsymbol\theta)
& \coloneqq
\sup_{\ell\in\mathcal L_\infty(L_{\max})}
\bigl|\mathcal R_{\mathrm{gen}}^{\ell}(\boldsymbol\theta)-\mathcal R_{\mathrm{disc}}^{\ell}(\boldsymbol\theta)\bigr| \\
& = 2L_{\max}\,\mathrm{TV}\!\bigl(P_{\boldsymbol\theta},Q\bigr).
\end{align}
\end{subequations}
Consequently,
\begin{equation}
\label{eq:iff_joint_app}
\Delta_\infty(\boldsymbol\theta)=0
\quad\Longleftrightarrow\quad
P_{\boldsymbol\theta}=Q.
\end{equation}

\smallskip
\noindent\textbf{(iv) Conditional total-variation decomposition under domination.}
Assume moreover that there exists a $\sigma$-finite measure $\lambda_{\mathcal Y}$ on $(\mathcal Y,\Sigma_{\mathcal Y})$
such that, for $\rho$-almost every $x$, both kernels are absolutely continuous w.r.t.\ $\lambda_{\mathcal Y}$, 
%
with jointly measurable densities $p_{\boldsymbol\theta}(\cdot\mid x)$ and $r(\cdot\mid x)$:
\begin{subequations}
\begin{align}
\pi_{\boldsymbol\theta}(\mathrm dy\mid x) & =p_{\boldsymbol\theta}(y\mid x)\,\lambda_{\mathcal Y}(\mathrm dy),\\ 
q(\mathrm dy\mid x) & = r(y\mid x)\,\lambda_{\mathcal Y}(\mathrm dy).    
\end{align}
\end{subequations}
Then
\begin{subequations}
\begin{align}
\label{eq:tv_conditional}
\mathrm{TV}\!\bigl(P_{\boldsymbol\theta},Q\bigr)
& =
\mathbb E_{x\sim\rho}\!\Bigl[\mathrm{TV}\!\bigl(\pi_{\boldsymbol\theta}(\cdot\mid x),\,q(\cdot\mid x)\bigr)\Bigr]\\
& =
\frac12\,\mathbb E_{x\sim\rho}\!\Bigl[\int_{\mathcal Y}\!\bigl|p_{\boldsymbol\theta}(y\mid x)-r(y\mid x)\bigr|
\,\lambda_{\mathcal Y}(\mathrm dy)\Bigr].
\end{align}
\end{subequations}
In particular,
\begin{equation}
\label{eq:iff_kernels}
\Delta_\infty(\boldsymbol\theta) \!= \! 0
\;\; \Longleftrightarrow\;\;
\pi_{\boldsymbol\theta}(\cdot \!\mid x)=q(\cdot \!\mid x)\ \text{for }\rho\text{-almost every\ }x.
\end{equation}

\end{proposition}

\begin{proof}
\leavevmode

\noindent
\textbf{(i).}
Because $\ell$ is bounded, it is integrable with respect to both $P_{\boldsymbol\theta}$ and $Q$. Using the disintegrations in \eqref{eq:joints_common_marginal_onpolicy_app} and \eqref{eq:joints_common_marginal_offpolicy_app}, and Tonelli's theorem,
\[
\mathcal R_{\mathrm{gen}}^{\ell}(\boldsymbol\theta)
=
\int_{\mathcal X}\rho(\mathrm dx)
\int_{\mathcal Y}\ell(x,y)\,\pi_{\boldsymbol\theta}(\mathrm dy\mid x),
\]
and
\[
\mathcal R_{\mathrm{disc}}^{\ell}(\boldsymbol\theta)
=
\int_{\mathcal X}\rho(\mathrm dx)
\int_{\mathcal Y}\ell(x,y)\,q(\mathrm dy\mid x).
\]
Subtracting these two expressions yields and taking absolute values gives \eqref{eq:blind-gap}.
If $\pi_{\boldsymbol\theta}(\cdot\mid x)=q(\cdot\mid x)$ for $\rho$-a.e.\ $x$, the inner difference vanishes $\rho$-a.e.,
hence $\Delta_\ell(\boldsymbol\theta)=0$.

\smallskip
\noindent\textbf{(ii).}
Let $f:\mathcal X\times\mathcal Y\to\mathbb R$ be measurable with $\|f\|_\infty\le 1$.
By the standard variational characterization of total variation,
\begin{equation}
\label{eq:tv_variational_app}
\sup_{\|f\|_\infty\le 1}\bigl|\mathbb E_{P_{\boldsymbol\theta}}[f]-\mathbb E_{Q}[f]\bigr|
=
2\,\mathrm{TV}(P_{\boldsymbol\theta},Q).
\end{equation}
Now let $0\le \ell\le L_{\max}$ and define $f=\ell/L_{\max}$. Then $\|f\|_\infty\le 1$, so
\begin{subequations}
\begin{align}
\bigl|
\mathbb E_{P_{\boldsymbol\theta}}[\ell]
-
\mathbb E_Q[\ell]
\bigr|
& =
L_{\max}
\bigl|
\mathbb E_{P_{\boldsymbol\theta}}[f]
-
\mathbb E_Q[f]
\bigr| \\
& \le
2L_{\max}\,\mathrm{TV}(P_{\boldsymbol\theta},Q),
\end{align}
\end{subequations}
which is \eqref{eq:tv_gap_app}.

\smallskip
\noindent
\textbf{(iii).}
Taking the supremum of the left-hand side of \eqref{eq:tv_gap_app} over all $\ell\in\mathcal L_\infty(L_{\max})$ and applying \eqref{eq:tv_variational_app} yields
\begin{subequations}
\begin{align}
\Delta_\infty(\boldsymbol\theta)
& =
L_{\max}
\sup_{\|f\|_\infty\le 1}
\bigl|
\mathbb E_{P_{\boldsymbol\theta}}[f]
-
\mathbb E_Q[f]
\bigr|\\ 
& =
2L_{\max}\,\mathrm{TV}(P_{\boldsymbol\theta},Q),
\end{align}
\end{subequations}
which proves \eqref{eq:delta_infty_def_app}. Since total variation vanishes if and only if the two probability measures are equal, \eqref{eq:iff_joint_app} follows.

\smallskip
\noindent
\textbf{(iv).}
Under the domination assumption, both $P_{\boldsymbol\theta}$ and $Q$ are absolutely continuous with respect to
the product measure $\rho\otimes\lambda_{\mathcal Y}$ on $\mathcal X\times\mathcal Y$, with densities
$p_{\boldsymbol\theta}(y\mid x)$ and $r(y\mid x)$, respectively:
\begin{subequations}
\begin{align}
P_{\boldsymbol\theta}(\mathrm dx,\mathrm dy) &= p_{\boldsymbol\theta}(y\mid x)\,(\rho\otimes\lambda_{\mathcal Y})(\mathrm dx,\mathrm dy), \\
Q(\mathrm dx,\mathrm dy) &= r(y\mid x)\,(\rho\otimes\lambda_{\mathcal Y})(\mathrm dx,\mathrm dy).
\end{align}   
\end{subequations}
Using the density form $\mathrm{TV}(\mu,\nu)=\tfrac12\int |\mathrm d\mu-\mathrm d\nu|$ under absolute continuity,
\begin{align}
\mathrm{TV}(P_{\boldsymbol\theta},Q)
&=
\frac12\int_{\mathcal X\times\mathcal Y} \!\! \! \bigl|p_{\boldsymbol\theta}(y\!\mid\! x)-r(y\!\mid\! x)\bigr|\,
(\rho\otimes\lambda_{\mathcal Y})(\mathrm dx,\mathrm dy) \nonumber \\
&=
\frac12\int_{\mathcal X}\rho(\mathrm dx)\int_{\mathcal Y}\bigl|p_{\boldsymbol\theta}(y\mid x)-r(y\mid x)\bigr|\,
\lambda_{\mathcal Y}(\mathrm dy) \nonumber \\
&=
\mathbb E_{x\sim\rho}\Bigl[\mathrm{TV}\bigl(\pi_{\boldsymbol\theta}(\cdot \! \mid x),q(\cdot \!\mid x)\bigr)\Bigr],
\end{align}
which is~\eqref{eq:tv_conditional}. Finally, $\Delta_\infty(\boldsymbol\theta)=0$ iff $\mathrm{TV}(P_{\boldsymbol\theta},Q)=0$, which holds iff the integrand in~\eqref{eq:tv_conditional} is zero $\rho$-a.e.,
i.e., iff $\mathrm{TV}(\pi_{\boldsymbol\theta}(\cdot \!\mid x),q(\cdot \!\mid x))=0$ for $\rho$-a.e.\ $x$, equivalently $\pi_{\boldsymbol\theta}(\cdot \!\mid x)=q(\cdot \!\mid x)$ for $\rho$-a.e.\ $x$. This proves~\eqref{eq:iff_kernels}.
\end{proof}

By Proposition~\ref{prop:blind-spot-gap}, we show a basic limitation of purely off-policy audits, i.e., for any fixed misalignment loss $\ell$, the off-policy risk $\mathcal R_{\mathrm{disc}}^{\ell}(\boldsymbol\theta)$ evaluates $\ell$ under the behavioral joint $Q=\rho\otimes q$, whereas deployment behavior is governed by the on-policy joint $P_{\boldsymbol\theta}=\rho\otimes\pi_{\boldsymbol\theta}$. Unless $Q$ matches the model-induced conditional law $\pi_{\boldsymbol\theta}(\cdot \!\mid x)$ for $\rho$-a.e.\ prompt, off-policy measurements can differ from the true on-policy risk by an amount controlled (in the worst case) by $\mathrm{TV}(P_{\boldsymbol\theta},Q)$. Consequently, audits that only score fixed candidate completions $(y_0, y_1)$ for each prompt $x$ can miss failure modes that arise under free generation, where the model must select a continuation distribution and thereby instantiate the principles in context.

\end{document}